\documentclass[a4paper,12pt]{amsart}

\usepackage{graphicx,natbib,amssymb,datetime,float,MnSymbol}
\usepackage{tikz}
\usetikzlibrary{arrows}

\newtheorem{theorem}{Theorem}

\newtheorem{corollary}{Corollary}

\theoremstyle{remark}

\def\ci{\!\perp\!}

\def\ra{\rightarrow}
\def\la{\leftarrow}

\newcommand{\comments}[1]{}

\begin{document}

\title[]{Error AMP Chain Graphs}

\author[]{Jose M. Pe\~{n}a\\
ADIT, IDA, Link\"oping University, SE-58183 Link\"{o}ping, Sweden\\
jose.m.pena@liu.se}

\date{\currenttime, \ddmmyydate{\today}}

\begin{abstract}
Any regular Gaussian probability distribution that can be represented by an AMP chain graph (CG) can be expressed as a system of linear equations with correlated errors whose structure depends on the CG. However, the CG represents the errors implicitly, as no nodes in the CG correspond to the errors. We propose in this paper to add some deterministic nodes to the CG in order to represent the errors explicitly. We call the result an EAMP CG. We will show that, as desired, every AMP CG is Markov equivalent to its corresponding EAMP CG under marginalization of the error nodes. We will also show that every EAMP CG under marginalization of the error nodes is Markov equivalent to some LWF CG under marginalization of the error nodes, and that the latter is Markov equivalent to some directed and acyclic graph (DAG) under marginalization of the error nodes and conditioning on some selection nodes. This is important because it implies that the independence model represented by an AMP CG can be accounted for by some data generating process that is partially observed and has selection bias. Finally, we will show that EAMP CGs are closed under marginalization. This is a desirable feature because it guarantees parsimonious models under marginalization.
\end{abstract}

\maketitle

\section{Introduction}

Chain graphs (CGs) are graphs with possibly directed and undirected edges, and no semidirected cycle. They have been extensively studied as a formalism to represent independence models. CGs extend Markov networks, i.e. undirected graphs, and Bayesian networks, i.e. directed and acyclic graphs (DAGs). Therefore, they can model symmetric and asymmetric relationships between the random variables of interest, which is one of the reasons of their popularity. However, unlike Markov and Bayesian networks whose interpretation is unique, there are four different interpretations of CGs as independence models \citep{CoxandWermuth1993,CoxandWermuth1996,Drton2009,SonntagandPenna2013}. In this paper, we are interested in the AMP interpretation \citep{Anderssonetal.2001,Levitzetal.2001} and the LWF interpretation \citep{Frydenberg1990,LauritzenWermuth1989}.

Any regular Gaussian probability distribution that can be represented by an AMP CG can be expressed as a system of linear equations with correlated errors whose structure depends on the CG \citep[Section 5]{Anderssonetal.2001}. However, the CG represents the errors implicitly, as no nodes in the CG correspond to the errors. We propose in this paper to add some deterministic nodes to the CG in order to represent the errors explicitly. We call the result an EAMP CG. We will show that, as desired, every AMP CG is Markov equivalent to its corresponding EAMP CG under marginalization of the error nodes, i.e. the independence model represented by the former coincides with the independence model represented by the latter. We will also show that every EAMP CG under marginalization of the error nodes is Markov equivalent to some LWF CG under marginalization of the error nodes, and that the latter is Markov equivalent to some DAG under marginalization of the error nodes and conditioning on some selection nodes. The relevance of this result can be best explained by extending to AMP CGs what \citet[p. 838]{Koster2002} stated for summary graphs and \citet[p. 981]{RichardsonandSpirtes2002} stated for ancestral graphs: The fact that an AMP CG has a DAG as departure point implies that the independence model associated with the former can be accounted for by some data generating process that is partially observed (corresponding to marginalization) and has selection bias (corresponding to conditioning). Finally, we will show that EAMP CGs are closed under marginalization, in the sense that every EAMP CG under marginalization of any superset of the error nodes is Markov equivalent to some EAMP CG under marginalization of the error nodes.\footnote{Our definition of closed under marginalization is an adaptation of the standard one to the fact that we only care about independence models under marginalization of the error nodes.} The relevance of this result can be best appreciated by noting that AMP CGs are not closed under marginalization \cite[Section 9.4]{RichardsonandSpirtes2002}. Therefore, the independence model represented by an AMP CG under marginalization may not be representable by any AMP CG. Therefore, we may have to represent it by an AMP CG with extra edges so as to avoid representing false independencies. However, if we consider the EAMP CG corresponding to the original AMP CG, then we will show that the marginal independence model can be represented by some EAMP CG under marginalization of the error nodes. The latter case is of course preferred, because the graphical model is more parsimonious as it does not include extra edges. See also \citet[p. 965]{RichardsonandSpirtes2002} for a discussion on the importance of the class of models considered being closed under marginalization.

It is worth mentioning that \citet[Theorem 6]{Anderssonetal.2001} have identified the conditions under which an AMP CG is Markov equivalent to some LWF CG.\footnote{To be exact, \citet[Theorem 6]{Anderssonetal.2001} have identified the conditions under which all and only the probability distributions that can be represented by an AMP CG can also be represented by some LWF CG. However, for any AMP or LWF CG $G$, there are Gaussian probability distributions that have all and only the independencies in the independence model represented by $G$, as shown by \citet[Theorem 6.1]{Levitzetal.2001} and \citet[Theorems 1 and 2]{Penna2011}. Then, our formulation is equivalent to the original formulation of the result by \citet[Theorem 6]{Anderssonetal.2001}.} It is clear from these conditions that there are AMP CGs that are not Markov equivalent to any LWF CG. The results in this paper differ from those by \citet[Theorem 6]{Anderssonetal.2001}, because we show that every AMP CG is Markov equivalent to some LWF CG with error nodes under marginalization of the error nodes.

It is also worth mentioning that \citet[p. 1025]{RichardsonandSpirtes2002} show that there are AMP CGs that are not Markov equivalent to any DAG under marginalization and conditioning. However, the results in this paper show that every AMP CG is Markov equivalent to some DAG with error and selection nodes under marginalization of the error nodes and conditioning of the selection nodes. Therefore, the independence model represented by any AMP CG has indeed some DAG as departure point and, thus, it can be accounted for by some data generating process. The results in this paper do not contradict those by \citet[p. 1025]{RichardsonandSpirtes2002}, because they did not consider deterministic nodes while we do (recall that the error nodes are deterministic).

Finally, it is also worth mentioning that EAMP CGs are not the first graphical models to have DAGs as departure point or to be closed under marginalization. Specifically, summary graphs \citep{CoxandWermuth1996}, MC graphs \citep{Koster2002}, ancestral graphs \citep{RichardsonandSpirtes2002}, and ribonless graphs \citep{Sadeghi2013} predate EAMP CGs and have the mentioned properties. However, none of these other classes of graphical models subsumes AMP CGs, i.e. there are independence models that can be represented by an AMP CG but not by any member of the other class \citep[Section 4]{SadeghiandLauritzen2012}. Therefore, none of these other classes of graphical models subsumes EAMP CGs under marginalization of the error nodes. This justifies the present study.

The rest of the paper is organized as follows. We start by reviewing some concepts in Section \ref{sec:preliminaries}. We discuss in Section \ref{sec:deterministic} the semantics of deterministic nodes in the context of AMP and LWF CGs. In Section \ref{sec:lwf}, we introduce EAMP CGs and use them to show that every AMP CG is Markov equivalent to some LWF CG under marginalization. In that section we also show that every AMP CG is Markov equivalent to some DAG under marginalization and conditioning. In Section \ref{sec:closed}, we show that EAMP CGs are closed under marginalization. Finally, we close with some conclusions in Section \ref{sec:conclusions}.

\section{Preliminaries}\label{sec:preliminaries}

In this section, we review some concepts from graphical models that are used later in this paper. All the graphs and probability distributions in this paper are defined over a finite set $V$ unless otherwise stated. The elements of $V$ are not distinguished from singletons. The operators set union and set difference are given equal precedence in the expressions. The term maximal is always wrt set inclusion. All the graphs in this paper are simple, i.e. they contain at most one edge between any pair of nodes. Moreover, the edge is undirected or directed.

If a graph $G$ contains an undirected or directed edge between two nodes $V_{1}$ and $V_{2}$, then we say that $V_{1} - V_{2}$ or $V_{1} \ra V_{2}$ is in $G$. The parents of a set of nodes $X$ of $G$ is the set $pa_G(X) = \{V_1 | V_1 \ra V_2$ is in $G$, $V_1 \notin X$ and $V_2 \in X \}$. A route between a node $V_{1}$ and a node $V_{n}$ in $G$ is a sequence of (not necessarily distinct) nodes $V_{1}, \ldots, V_{n}$ st $V_i - V_{i+1}$, $V_i \ra V_{i+1}$ or $V_i \la V_{i+1}$ is in $G$ for all $1 \leq i < n$. If the nodes in the route are all distinct, then the route is called a path. A route is called undirected if $V_i - V_{i+1}$ is in $G$ for all $1 \leq i < n$. A route is called strictly descending if $V_i \ra V_{i+1}$ is in $G$ for all $1 \leq i < n$. The strict ascendants of $X$ is the set $san_G(X) = \{V_1 |$ there is a strictly descending route from $V_1$ to $V_n$ in $G$, $V_1 \notin X$ and $V_n \in X \}$. A route $V_{1}, \ldots, V_{n}$ in $G$ is called a cycle if $V_n=V_1$. Moreover, it is called a semidirected cycle if $V_1 \ra V_2$ is in $G$ and $V_i \ra V_{i+1}$ or $V_i - V_{i+1}$ is in $G$ for all $1 < i < n$. A chain graph (CG) is a graph with no semidirected cycles. A set of nodes of a graph is connected if there exists an undirected path in the graph between every pair of nodes in the set. A connectivity component of a CG is a maximal connected set.

We now recall the semantics of AMP and LWF CGs. A node $B$ in a path $\rho$ in an AMP CG $G$ is called a triplex node in $\rho$ if $A \ra B \la C$, $A \ra B - C$, or $A - B \la C$ is a subpath of $\rho$. Moreover, $\rho$ is said to be $Z$-open with $Z \subseteq V$ when

\begin{itemize}
\item every triplex node in $\rho$ is in $Z \cup san_G(Z)$, and

\item no non-triplex node $B$ in $\rho$ is in $Z$, unless $A - B - C$ is a subpath of $\rho$ and some node in $pa_G(B)$ is not in $Z$.
\end{itemize}

A section of a route $\rho$ in a CG is a maximal undirected subroute of $\rho$. A section $V_{2} - \ldots - V_{n-1}$ of $\rho$ is a collider section of $\rho$ if $V_{1} \rightarrow V_{2} - \ldots - V_{n-1} \leftarrow V_{n}$ is a subroute of $\rho$. A route $\rho$ in a CG is said to be $Z$-open when

\begin{itemize}
\item every collider section of $\rho$ has a node in $Z$, and

\item no non-collider section of $\rho$ has a node in $Z$.
\end{itemize}

Let $X$, $Y$ and $Z$ denote three disjoint subsets of $V$. When there is no $Z$-open path (respectively route) in an AMP (respectively LWF) CG $G$ between a node in $X$ and a node in $Y$, we say that $X$ is separated from $Y$ given $Z$ in $G$ and denote it as $X \ci_G Y | Z$. The independence model represented by $G$, denoted as $I_{AMP}(G)$ or $I_{LWF}(G)$, is the set of separations $X \ci_G Y | Z$. In general, $I_{AMP}(G) \neq I_{LWF}(G)$. However, if $G$ is a directed and acyclic graph (DAG), then $I_{AMP}(G) = I_{LWF}(G)$. Given an AMP or LWF CG $G$ and two disjoint subsets $L$ and $S$ of $V$, we denote by $[I(G)]_L^S$ the independence model represented by $G$ under marginalization of the nodes in $L$ and conditioning on the nodes in $S$. Specifically, $X \ci_G Y | Z$ is in $[I(G)]_L^S$ iff $X \ci_G Y | Z \cup S$ is in $I(G)$ and $X, Y, Z \subseteq V \setminus L \setminus S$.

Finally, we denote by $X \ci_p Y | Z$ that $X$ is independent of $Y$ given $Z$ in a probability distribution $p$. We say that $p$ is Markovian wrt an AMP or LWF CG $G$ when $X \ci_p Y | Z$ if $X \ci_G Y | Z$ for all $X$, $Y$ and $Z$ disjoint subsets of $V$. We say that $p$ is faithful to $G$ when $X \ci_p Y | Z$ iff $X \ci_G Y | Z$ for all $X$, $Y$ and $Z$ disjoint subsets of $V$.

\section{AMP and LWF CGs with Deterministic Nodes}\label{sec:deterministic}

We say that a node $A$ of an AMP or LWF CG is determined by some $Z \subseteq V$ when $A \in Z$ or $A$ is a function of $Z$. In that case, we also say that $A$ is a deterministic node. We use $D(Z)$ to denote all the nodes that are determined by $Z$. From the point of view of the separations in an AMP or LWF CG, that a node is determined by but is not in the conditioning set of a separation has the same effect as if the node were actually in the conditioning set. We extend the definitions of separation for AMP and LWF CGs to the case where deterministic nodes may exist.

Given an AMP CG $G$, a path $\rho$ in $G$ is said to be $Z$-open when

\begin{itemize}
\item every triplex node in $\rho$ is in $D(Z) \cup san_G(D(Z))$, and

\item no non-triplex node $B$ in $\rho$ is in $D(Z)$, unless $A - B - C$ is a subpath of $\rho$ and some node in $pa_G(B)$ is not in $D(Z)$.
\end{itemize}

Given an LWF CG $G$, a route $\rho$ in $G$ is said to be $Z$-open when

\begin{itemize}
\item every collider section of $\rho$ has a node in $D(Z)$, and

\item no non-collider section of $\rho$ has a node in $D(Z)$.
\end{itemize}

It should be noted that we are not the first to consider graphical models with deterministic nodes. For instance, \citet[Section 4]{Geigeretal.1990} consider DAGs with deterministic nodes. However, our definition of deterministic node is more general than theirs.

\section{From AMP CGs to DAGs Via EAMP CGs}\label{sec:lwf}

\citet[Section 5]{Anderssonetal.2001} show that any regular Gaussian probability distribution $p$ that is Markovian wrt an AMP CG $G$ can be expressed as a system of linear equations with correlated errors whose structure depends on $G$. Specifically, assume without loss of generality that $p$ has mean 0. Let $K_i$ denote any connectivity component of $G$. Let $\Omega^i_{K_i,K_i}$ and $\Omega^i_{K_i ,pa_G(K_i)}$ denote submatrices of the precision matrix $\Omega^i$ of $p(K_i, pa_G(K_i))$. Then, as shown by \citet[Section 2.3.1]{Bishop2006},
\[
K_i | pa_G(K_i) \sim \mathcal{N}(\beta^i pa_G(K_i), \Lambda^i)
\]
where
\[
\beta^i= -(\Omega^i_{K_i,K_i})^{-1} \Omega^i_{K_i ,pa_G(K_i)}
\]
and
\[
(\Lambda^i)^{-1}= \Omega^i_{K_i,K_i}.
\]

Then, $p$ can be expressed as a system of linear equations with normally distributed errors whose structure depends on $G$ as follows:
\[
K_i = \beta^i \: pa_G(K_i) + \epsilon^i
\]
where 
\[
\epsilon^i \sim \mathcal{N}(0, \Lambda^i).
\]

Note that for all $A, B \in K_i$ st $A- B$ is not in $G$, $A \ci_G B | pa_G(K_i) \cup K_i \setminus A \setminus B$ and thus $(\Lambda^i)^{-1}_{A,B} = 0$ \citep[Proposition 5.2]{Lauritzen1996}. Note also that for all $A \in K_i$ and $B \in pa_G(K_i)$ st $A \la B$ is not in $G$, $A \ci_G B | pa_G(A)$ and thus $(\beta^i)_{A,B}=0$. Let $\beta_A$ contain the nonzero elements of the vector $(\beta^i)_{A, \bullet}$. Then, $p$ can be expressed as a system of linear equations with correlated errors whose structure depends on $G$ as follows. For any $A \in K_i$,
\[
A = \beta_A \: pa_G(A) + \epsilon^A
\]
and for any other $B \in K_i$,
\[
covariance(\epsilon^A, \epsilon^B) = \Lambda^i_{A,B}.
\]

It is worth mentioning that the mapping above between probability distributions and systems of linear equations is bijective \citep[Section 5]{Anderssonetal.2001}. Note that no nodes in $G$ correspond to the errors $\epsilon^A$. Therefore, $G$ represent the errors implicitly. We propose to represent them explicitly. This can easily be done by transforming $G$ into what we call an EAMP CG $G'$ as follows:

\begin{table}[H]
\centering
\scalebox{1.0}{
\begin{tabular}{ll}
1 & Let $G'=G$\\
2 & For each node $A$ in $G$\\
3 & \hspace{0.3cm} Add the node $\epsilon^A$ to $G'$\\
4 & \hspace{0.3cm} Add the edge $\epsilon^A \ra A$ to $G'$\\
5 & For each edge $A - B$ in $G$\\
6 & \hspace{0.3cm} Add the edge $\epsilon^A - \epsilon^B$ to $G'$\\
7 & \hspace{0.3cm} Remove the edge $A - B$ from $G'$\\
\end{tabular}}
\end{table}

The transformation above basically consists in adding the error nodes $\epsilon^A$ to $G$ and connect them appropriately. Figure \ref{fig:example} shows an example. Note that every node $A \in V$ is determined by $pa_{G'}(A)$ and, what is more important in this paper, that $\epsilon^A$ is determined by $pa_{G'}(A) \setminus \epsilon^A \cup A$. Note also that, given $Z \subseteq V$, a node $A \in V$ is determined by $Z$ iff $A \in Z$. The if part is trivial. To see the only if part, note that $\epsilon^A \notin Z$ and thus $A$ cannot be determined by $Z$ unless $A \in Z$. Therefore, a node $\epsilon^A$ in $G'$ is determined by $Z$ iff $pa_{G'}(A) \setminus \epsilon^A \cup A \subseteq Z$ because, as shown, there is no other way for $Z$ to determine $pa_{G'}(A) \setminus \epsilon^A \cup A$ which, in turn, determine $\epsilon^A$. Let $\epsilon$ denote all the error nodes in $G'$. It is easy to see that $G'$ is an AMP CG over $V \cup \epsilon$ and, thus, its semantics are defined. The following theorem confirms that these semantics are as desired.

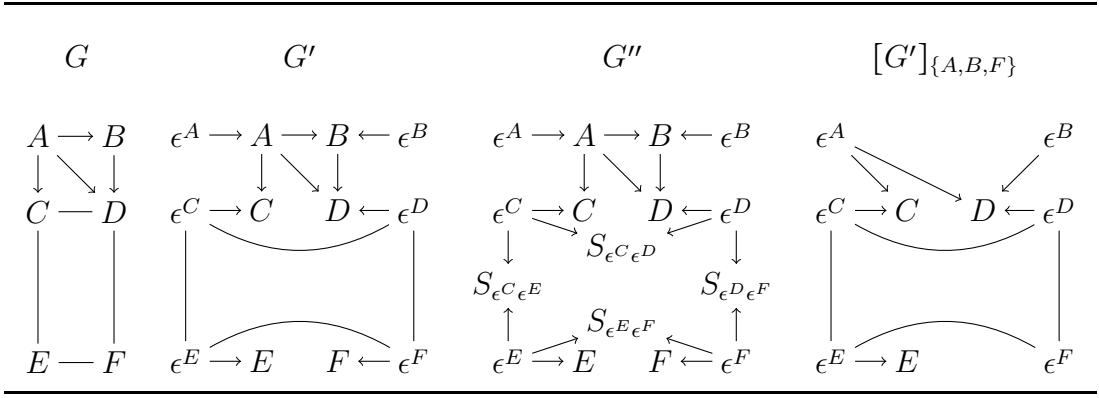
\begin{figure}
\centering
\scalebox{1.0}{
\begin{tabular}{cccc}\\
\hline
\\
$G$&$G'$&$G''$&$[G']_{\{A,B,F\}}$\\
\\
\begin{tikzpicture}[inner sep=1mm]
\node at (0,0) (A) {$A$};
\node at (1,0) (B) {$B$};
\node at (0,-1) (C) {$C$};
\node at (1,-1) (D) {$D$};
\node at (0,-3) (E) {$E$};
\node at (1,-3) (F) {$F$};
\path[->] (A) edge (B);
\path[->] (A) edge (C);
\path[->] (A) edge (D);
\path[->] (B) edge (D);
\path[-] (C) edge (D);
\path[-] (C) edge (E);
\path[-] (D) edge (F);
\path[-] (E) edge (F);
\end{tikzpicture}
&
\begin{tikzpicture}[inner sep=1mm]
\node at (0,0) (A) {$A$};
\node at (1,0) (B) {$B$};
\node at (0,-1) (C) {$C$};
\node at (1,-1) (D) {$D$};
\node at (0,-3) (E) {$E$};
\node at (1,-3) (F) {$F$};
\node at (-1,0) (EA) {$\epsilon^A$};
\node at (2,0) (EB) {$\epsilon^B$};
\node at (-1,-1) (EC) {$\epsilon^C$};
\node at (2,-1) (ED) {$\epsilon^D$};
\node at (-1,-3) (EE) {$\epsilon^E$};
\node at (2,-3) (EF) {$\epsilon^F$};
\path[->] (EA) edge (A);
\path[->] (EB) edge (B);
\path[->] (EC) edge (C);
\path[->] (ED) edge (D);
\path[->] (EE) edge (E);
\path[->] (EF) edge (F);
\path[->] (A) edge (B);
\path[->] (A) edge (C);
\path[->] (A) edge (D);
\path[->] (B) edge (D);
\path[-] (EC) edge [bend right] (ED);
\path[-] (EC) edge (EE);
\path[-] (ED) edge (EF);
\path[-] (EE) edge [bend left] (EF);
\end{tikzpicture}
&
\begin{tikzpicture}[inner sep=1mm]
\node at (0,0) (A) {$A$};
\node at (1,0) (B) {$B$};
\node at (0,-1) (C) {$C$};
\node at (1,-1) (D) {$D$};
\node at (0,-3) (E) {$E$};
\node at (1,-3) (F) {$F$};
\node at (-1,0) (EA) {$\epsilon^A$};
\node at (2,0) (EB) {$\epsilon^B$};
\node at (-1,-1) (EC) {$\epsilon^C$};
\node at (2,-1) (ED) {$\epsilon^D$};
\node at (-1,-3) (EE) {$\epsilon^E$};
\node at (2,-3) (EF) {$\epsilon^F$};
\node at (0.5,-1.5) (SCD) {$S_{\epsilon^C\epsilon^D}$};
\node at (-1,-2) (SCE) {$S_{\epsilon^C\epsilon^E}$};
\node at (2,-2) (SDF) {$S_{\epsilon^D\epsilon^F}$};
\node at (0.5,-2.5) (SEF) {$S_{\epsilon^E\epsilon^F}$};
\path[->] (EA) edge (A);
\path[->] (EB) edge (B);
\path[->] (EC) edge (C);
\path[->] (ED) edge (D);
\path[->] (EE) edge (E);
\path[->] (EF) edge (F);
\path[->] (A) edge (B);
\path[->] (A) edge (C);
\path[->] (A) edge (D);
\path[->] (B) edge (D);
\path[->] (EC) edge (SCE);
\path[->] (EE) edge (SCE);
\path[->] (ED) edge (SDF);
\path[->] (EF) edge (SDF);
\path[->] (EC) edge (SCD);
\path[->] (ED) edge (SCD);
\path[->] (EE) edge (SEF);
\path[->] (EF) edge (SEF);
\end{tikzpicture}
&
\begin{tikzpicture}[inner sep=1mm]
\node at (0,-1) (C) {$C$};
\node at (1,-1) (D) {$D$};
\node at (0,-3) (E) {$E$};
\node at (-1,0) (EA) {$\epsilon^A$};
\node at (2,0) (EB) {$\epsilon^B$};
\node at (-1,-1) (EC) {$\epsilon^C$};
\node at (2,-1) (ED) {$\epsilon^D$};
\node at (-1,-3) (EE) {$\epsilon^E$};
\node at (2,-3) (EF) {$\epsilon^F$};
\path[->] (EA) edge (C);
\path[->] (EA) edge (D);
\path[->] (EB) edge (D);
\path[->] (EC) edge (C);
\path[->] (ED) edge (D);
\path[->] (EE) edge (E);
\path[-] (EC) edge [bend right] (ED);
\path[-] (EC) edge (EE);
\path[-] (ED) edge (EF);
\path[-] (EE) edge [bend left] (EF);
\end{tikzpicture}\\
\hline
\end{tabular}}\caption{Example of the different transformations.}\label{fig:example}
\end{figure}

\begin{theorem}\label{the:GG'}
$I_{AMP}(G)=[I_{AMP}(G')]_\epsilon^\emptyset$.
\end{theorem}

\begin{proof}
It suffices to show that every $Z$-open path between $\alpha$ and $\beta$ in $G$ can be transformed into a $Z$-open path between $\alpha$ and $\beta$ in $G'$ and vice versa, with $\alpha, \beta \in V$ and $Z \subseteq V \setminus \alpha \setminus \beta$.

Let $\rho$ denote a $Z$-open path between $\alpha$ and $\beta$ in $G$. We can easily transform $\rho$ into a path $\rho'$ between $\alpha$ and $\beta$ in $G'$: Simply, replace every maximal subpath of $\rho$ of the form $V_1 - V_2 - \ldots - V_{n-1} - V_n$ ($n \geq 2$) with $V_1 \la \epsilon^{V_1} - \epsilon^{V_2} - \ldots - \epsilon^{V_{n-1}} - \epsilon^{V_n} \ra V_n$. We now show that $\rho'$ is $Z$-open.

First, if $B \in V$ is a triplex node in $\rho'$, then $\rho'$ must have one of the following subpaths:

\begin{table}[H]
\centering
\scalebox{1.0}{
\begin{tabular}{c}
\begin{tikzpicture}[inner sep=1mm]
\node at (0,0) (A) {$A$};
\node at (1,0) (B) {$B$};
\node at (2,0) (C) {$C$};
\path[->] (A) edge (B);
\path[<-] (B) edge (C);
\end{tikzpicture}
\begin{tikzpicture}[inner sep=1mm]
\node at (0,0) (A) {$A$};
\node at (1,0) (B) {$B$};
\node at (2,0) (C) {$\epsilon^B$};
\node at (3,0) (D) {$\epsilon^C$};
\path[->] (A) edge (B);
\path[<-] (B) edge (C);
\path[-] (D) edge (C);
\end{tikzpicture}
\begin{tikzpicture}[inner sep=1mm]
\node at (0,0) (A) {$\epsilon^B$};
\node at (1,0) (B) {$B$};
\node at (2,0) (C) {$C$};
\node at (-1,0) (D) {$\epsilon^A$};
\path[->] (A) edge (B);
\path[<-] (B) edge (C);
\path[-] (D) edge (A);
\end{tikzpicture}
\end{tabular}}
\end{table}

with $A, C \in V$. Therefore, $\rho$ must have one of the following subpaths (specifically, if $\rho'$ has the $i$-th subpath above, then $\rho$ has the $i$-th subpath below):

\begin{table}[H]
\centering
\scalebox{1.0}{
\begin{tabular}{c}
\begin{tikzpicture}[inner sep=1mm]
\node at (0,0) (A) {$A$};
\node at (1,0) (B) {$B$};
\node at (2,0) (C) {$C$};
\path[->] (A) edge (B);
\path[<-] (B) edge (C);
\end{tikzpicture}
\begin{tikzpicture}[inner sep=1mm]
\node at (0,0) (A) {$A$};
\node at (1,0) (B) {$B$};
\node at (2,0) (C) {$C$};
\path[->] (A) edge (B);
\path[-] (B) edge (C);
\end{tikzpicture}
\begin{tikzpicture}[inner sep=1mm]
\node at (0,0) (A) {$A$};
\node at (1,0) (B) {$B$};
\node at (2,0) (C) {$C$};
\path[-] (A) edge (B);
\path[<-] (B) edge (C);
\end{tikzpicture}
\end{tabular}}
\end{table}

In either case, $B$ is a triplex node in $\rho$ and, thus, $B \in Z \cup san_G(Z)$ for $\rho$ to be $Z$-open. Then, $B \in Z \cup san_{G'}(Z)$ by construction of $G'$ and, thus, $B \in D(Z) \cup san_{G'}(D(Z))$.

Second, if $B \in V$ is a non-triplex node in $\rho'$, then $\rho'$ must have one of the following subpaths:

\begin{table}[H]
\centering
\scalebox{1.0}{
\begin{tabular}{c}
\begin{tikzpicture}[inner sep=1mm]
\node at (0,0) (A) {$A$};
\node at (1,0) (B) {$B$};
\node at (2,0) (C) {$C$};
\path[->] (A) edge (B);
\path[->] (B) edge (C);
\end{tikzpicture}
\begin{tikzpicture}[inner sep=1mm]
\node at (0,0) (A) {$A$};
\node at (1,0) (B) {$B$};
\node at (2,0) (C) {$C$};
\path[<-] (A) edge (B);
\path[->] (B) edge (C);
\end{tikzpicture}
\begin{tikzpicture}[inner sep=1mm]
\node at (0,0) (A) {$A$};
\node at (1,0) (B) {$B$};
\node at (2,0) (C) {$C$};
\path[<-] (A) edge (B);
\path[<-] (B) edge (C);
\end{tikzpicture}
\begin{tikzpicture}[inner sep=1mm]
\node at (0,0) (A) {$A$};
\node at (1,0) (B) {$B$};
\node at (2,0) (C) {$\epsilon^B$};
\node at (3,0) (D) {$\epsilon^C$};
\path[<-] (A) edge (B);
\path[<-] (B) edge (C);
\path[-] (D) edge (C);
\end{tikzpicture}
\begin{tikzpicture}[inner sep=1mm]
\node at (0,0) (A) {$\epsilon^B$};
\node at (1,0) (B) {$B$};
\node at (2,0) (C) {$C$};
\node at (-1,0) (D) {$\epsilon^A$};
\path[->] (A) edge (B);
\path[->] (B) edge (C);
\path[-] (D) edge (A);
\end{tikzpicture}
\end{tabular}}
\end{table}

with $A, C \in V$. Therefore, $\rho$ must have one of the following subpaths (specifically, if $\rho'$ has the $i$-th subpath above, then $\rho$ has the $i$-th subpath below):

\begin{table}[H]
\centering
\scalebox{1.0}{
\begin{tabular}{c}
\begin{tikzpicture}[inner sep=1mm]
\node at (0,0) (A) {$A$};
\node at (1,0) (B) {$B$};
\node at (2,0) (C) {$C$};
\path[->] (A) edge (B);
\path[->] (B) edge (C);
\end{tikzpicture}
\begin{tikzpicture}[inner sep=1mm]
\node at (0,0) (A) {$A$};
\node at (1,0) (B) {$B$};
\node at (2,0) (C) {$C$};
\path[<-] (A) edge (B);
\path[->] (B) edge (C);
\end{tikzpicture}
\begin{tikzpicture}[inner sep=1mm]
\node at (0,0) (A) {$A$};
\node at (1,0) (B) {$B$};
\node at (2,0) (C) {$C$};
\path[<-] (A) edge (B);
\path[<-] (B) edge (C);
\end{tikzpicture}
\begin{tikzpicture}[inner sep=1mm]
\node at (0,0) (A) {$A$};
\node at (1,0) (B) {$B$};
\node at (2,0) (C) {$C$};
\path[<-] (A) edge (B);
\path[-] (B) edge (C);
\end{tikzpicture}
\begin{tikzpicture}[inner sep=1mm]
\node at (0,0) (A) {$A$};
\node at (1,0) (B) {$B$};
\node at (2,0) (C) {$C$};
\path[-] (A) edge (B);
\path[->] (B) edge (C);
\end{tikzpicture}
\end{tabular}}
\end{table}

In either case, $B$ is a non-triplex node in $\rho$ and, thus, $B \notin Z$ for $\rho$ to be $Z$-open. Since $Z$ contains no error node, $Z$ cannot determine any node in $V$ that is not already in $Z$. Then, $B \notin D(Z)$.

Third, if $\epsilon^B$ is a non-triplex node in $\rho'$ (note that $\epsilon^B$ cannot be a triplex node in $\rho'$), then $\rho'$ must have one of the following subpaths:

\begin{table}[H]
\centering
\scalebox{1.0}{
\begin{tabular}{c}
\begin{tikzpicture}[inner sep=1mm]
\node at (0,0) (A) {$A$};
\node at (1,0) (B) {$B$};
\node at (2,0) (C) {$\epsilon^B$};
\node at (3,0) (D) {$\epsilon^C$};
\path[->] (A) edge (B);
\path[<-] (B) edge (C);
\path[-] (D) edge (C);
\end{tikzpicture}
\begin{tikzpicture}[inner sep=1mm]
\node at (0,0) (A) {$\epsilon^B$};
\node at (1,0) (B) {$B$};
\node at (2,0) (C) {$C$};
\node at (-1,0) (D) {$\epsilon^A$};
\path[->] (A) edge (B);
\path[<-] (B) edge (C);
\path[-] (D) edge (A);
\end{tikzpicture}
\begin{tikzpicture}[inner sep=1mm]
\node at (0.65,0) (B) {$\alpha=B$};
\node at (2,0) (C) {$\epsilon^B$};
\node at (3,0) (D) {$\epsilon^C$};
\path[<-] (B) edge (C);
\path[-] (D) edge (C);
\end{tikzpicture}
\begin{tikzpicture}[inner sep=1mm]
\node at (0,0) (A) {$\epsilon^B$};
\node at (1.35,0) (B) {$B=\beta$};
\node at (-1,0) (D) {$\epsilon^A$};
\path[->] (A) edge (B);
\path[-] (D) edge (A);
\end{tikzpicture}\\
\begin{tikzpicture}[inner sep=1mm]
\node at (0,0) (A) {$A$};
\node at (1,0) (B) {$B$};
\node at (2,0) (C) {$\epsilon^B$};
\node at (3,0) (D) {$\epsilon^C$};
\path[<-] (A) edge (B);
\path[<-] (B) edge (C);
\path[-] (D) edge (C);
\end{tikzpicture}
\begin{tikzpicture}[inner sep=1mm]
\node at (0,0) (A) {$\epsilon^B$};
\node at (1,0) (B) {$B$};
\node at (2,0) (C) {$C$};
\node at (-1,0) (D) {$\epsilon^A$};
\path[->] (A) edge (B);
\path[->] (B) edge (C);
\path[-] (D) edge (A);
\end{tikzpicture}
\begin{tikzpicture}[inner sep=1mm]
\node at (0,0) (A) {$\epsilon^A$};
\node at (1,0) (B) {$\epsilon^B$};
\node at (2,0) (C) {$\epsilon^C$};
\path[-] (A) edge (B);
\path[-] (B) edge (C);
\end{tikzpicture}
\end{tabular}}
\end{table}

with $A, C \in V$. Recall that $\epsilon^B \notin Z$ because $Z \subseteq V \setminus \alpha \setminus \beta$. In the first case, if $\alpha=A$ then $A \notin Z$, else $A \notin Z$ for $\rho$ to be $Z$-open. Then, $\epsilon^B \notin D(Z)$. In the second case, if $\beta=C$ then $C \notin Z$, else $C \notin Z$ for $\rho$ to be $Z$-open. Then, $\epsilon^B \notin D(Z)$. In the third and fourth cases, $B \notin Z$ because $\alpha=B$ or $\beta=B$. Then, $\epsilon^B \notin D(Z)$. In the fifth and sixth cases, $B \notin Z$ for $\rho$ to be $Z$-open. Then, $\epsilon^B \notin D(Z)$. The last case implies that $\rho$ has the following subpath:

\begin{table}[H]
\centering
\scalebox{1.0}{
\begin{tabular}{c}
\begin{tikzpicture}[inner sep=1mm]
\node at (0,0) (A) {$A$};
\node at (1,0) (B) {$B$};
\node at (2,0) (C) {$C$};
\path[-] (A) edge (B);
\path[-] (B) edge (C);
\end{tikzpicture}
\end{tabular}}
\end{table}

Thus, $B$ is a non-triplex node in $\rho$, which implies that $B \notin Z$ or $pa_G(B) \setminus Z \neq \emptyset$ for $\rho$ to be $Z$-open. In either case, $\epsilon^B \notin D(Z)$ (recall that $pa_{G'}(B)=pa_G(B) \cup \epsilon^B$ by construction of $G'$).

Finally, let $\rho'$ denote a $Z$-open path between $\alpha$ and $\beta$ in $G'$. We can easily transform $\rho'$ into a path $\rho$ between $\alpha$ and $\beta$ in $G$: Simply, replace every maximal subpath of $\rho'$ of the form $V_1 \la \epsilon^{V_1} - \epsilon^{V_2} - \ldots - \epsilon^{V_{n-1}} - \epsilon^{V_n} \ra V_n$ ($n \geq 2$) with $V_1 - V_2 - \ldots - V_{n-1} - V_n$. We now show that $\rho$ is $Z$-open.

First, note that all the nodes in $\rho$ are in $V$. Moreover, if $B$ is a triplex node in $\rho$, then $\rho$ must have one of the following subpaths:

\begin{table}[H]
\centering
\scalebox{1.0}{
\begin{tabular}{c}
\begin{tikzpicture}[inner sep=1mm]
\node at (0,0) (A) {$A$};
\node at (1,0) (B) {$B$};
\node at (2,0) (C) {$C$};
\path[->] (A) edge (B);
\path[<-] (B) edge (C);
\end{tikzpicture}
\begin{tikzpicture}[inner sep=1mm]
\node at (0,0) (A) {$A$};
\node at (1,0) (B) {$B$};
\node at (2,0) (C) {$C$};
\path[->] (A) edge (B);
\path[-] (B) edge (C);
\end{tikzpicture}
\begin{tikzpicture}[inner sep=1mm]
\node at (0,0) (A) {$A$};
\node at (1,0) (B) {$B$};
\node at (2,0) (C) {$C$};
\path[-] (A) edge (B);
\path[<-] (B) edge (C);
\end{tikzpicture}
\end{tabular}}
\end{table}

with $A, C \in V$. Therefore, $\rho'$ must have one of the following subpaths (specifically, if $\rho$ has the $i$-th subpath above, then $\rho'$ has the $i$-th subpath below):

\begin{table}[H]
\centering
\scalebox{1.0}{
\begin{tabular}{c}
\begin{tikzpicture}[inner sep=1mm]
\node at (0,0) (A) {$A$};
\node at (1,0) (B) {$B$};
\node at (2,0) (C) {$C$};
\path[->] (A) edge (B);
\path[<-] (B) edge (C);
\end{tikzpicture}
\begin{tikzpicture}[inner sep=1mm]
\node at (0,0) (A) {$A$};
\node at (1,0) (B) {$B$};
\node at (2,0) (C) {$\epsilon^B$};
\node at (3,0) (D) {$\epsilon^C$};
\path[->] (A) edge (B);
\path[<-] (B) edge (C);
\path[-] (D) edge (C);
\end{tikzpicture}
\begin{tikzpicture}[inner sep=1mm]
\node at (0,0) (A) {$\epsilon^B$};
\node at (1,0) (B) {$B$};
\node at (2,0) (C) {$C$};
\node at (-1,0) (D) {$\epsilon^A$};
\path[->] (A) edge (B);
\path[<-] (B) edge (C);
\path[-] (D) edge (A);
\end{tikzpicture}
\end{tabular}}
\end{table}

In either case, $B$ is a triplex node in $\rho'$ and, thus, $B \in D(Z) \cup san_{G'}(D(Z))$ for $\rho'$ to be $Z$-open. Since $Z$ contains no error node, $Z$ cannot determine any node in $V$ that is not already in $Z$. Then, $B \in D(Z)$ iff $B \in Z$. Since there is no strictly descending route from $B$ to any error node, then any strictly descending route from $B$ to a node $D \in D(Z)$ implies that $D \in V$ which, as seen, implies that $D \in Z$. Then, $B \in san_{G'}(D(Z))$ iff $B \in san_{G'}(Z)$. Moreover, $B \in san_{G'}(Z)$ iff $B \in san_{G}(Z)$ by construction of $G'$. These results together imply that $B \in Z \cup san_{G}(Z)$.

Second, if $B$ is a non-triplex node in $\rho$, then $\rho$ must have one of the following subpaths:

\begin{table}[H]
\centering
\scalebox{1.0}{
\begin{tabular}{c}
\begin{tikzpicture}[inner sep=1mm]
\node at (0,0) (A) {$A$};
\node at (1,0) (B) {$B$};
\node at (2,0) (C) {$C$};
\path[->] (A) edge (B);
\path[->] (B) edge (C);
\end{tikzpicture}
\begin{tikzpicture}[inner sep=1mm]
\node at (0,0) (A) {$A$};
\node at (1,0) (B) {$B$};
\node at (2,0) (C) {$C$};
\path[<-] (A) edge (B);
\path[->] (B) edge (C);
\end{tikzpicture}
\begin{tikzpicture}[inner sep=1mm]
\node at (0,0) (A) {$A$};
\node at (1,0) (B) {$B$};
\node at (2,0) (C) {$C$};
\path[<-] (A) edge (B);
\path[<-] (B) edge (C);
\end{tikzpicture}
\begin{tikzpicture}[inner sep=1mm]
\node at (0,0) (A) {$A$};
\node at (1,0) (B) {$B$};
\node at (2,0) (C) {$C$};
\path[<-] (A) edge (B);
\path[-] (B) edge (C);
\end{tikzpicture}
\begin{tikzpicture}[inner sep=1mm]
\node at (0,0) (A) {$A$};
\node at (1,0) (B) {$B$};
\node at (2,0) (C) {$C$};
\path[-] (A) edge (B);
\path[->] (B) edge (C);
\end{tikzpicture}
\begin{tikzpicture}[inner sep=1mm]
\node at (0,0) (A) {$A$};
\node at (1,0) (B) {$B$};
\node at (2,0) (C) {$C$};
\path[-] (A) edge (B);
\path[-] (B) edge (C);
\end{tikzpicture}
\end{tabular}}
\end{table}

with $A, C \in V$. Therefore, $\rho'$ must have one of the following subpaths (specifically, if $\rho$ has the $i$-th subpath above, then $\rho'$ has the $i$-th subpath below):

\begin{table}[H]
\centering
\scalebox{1.0}{
\begin{tabular}{c}
\begin{tikzpicture}[inner sep=1mm]
\node at (0,0) (A) {$A$};
\node at (1,0) (B) {$B$};
\node at (2,0) (C) {$C$};
\path[->] (A) edge (B);
\path[->] (B) edge (C);
\end{tikzpicture}
\begin{tikzpicture}[inner sep=1mm]
\node at (0,0) (A) {$A$};
\node at (1,0) (B) {$B$};
\node at (2,0) (C) {$C$};
\path[<-] (A) edge (B);
\path[->] (B) edge (C);
\end{tikzpicture}
\begin{tikzpicture}[inner sep=1mm]
\node at (0,0) (A) {$A$};
\node at (1,0) (B) {$B$};
\node at (2,0) (C) {$C$};
\path[<-] (A) edge (B);
\path[<-] (B) edge (C);
\end{tikzpicture}
\begin{tikzpicture}[inner sep=1mm]
\node at (0,0) (A) {$A$};
\node at (1,0) (B) {$B$};
\node at (2,0) (C) {$\epsilon^B$};
\node at (3,0) (D) {$\epsilon^C$};
\path[<-] (A) edge (B);
\path[<-] (B) edge (C);
\path[-] (D) edge (C);
\end{tikzpicture}
\begin{tikzpicture}[inner sep=1mm]
\node at (0,0) (A) {$\epsilon^B$};
\node at (1,0) (B) {$B$};
\node at (2,0) (C) {$C$};
\node at (-1,0) (D) {$\epsilon^A$};
\path[->] (A) edge (B);
\path[->] (B) edge (C);
\path[-] (D) edge (A);
\end{tikzpicture}\\
\begin{tikzpicture}[inner sep=1mm]
\node at (0,0) (A) {$\epsilon^A$};
\node at (1,0) (B) {$\epsilon^B$};
\node at (2,0) (C) {$\epsilon^C$};
\path[-] (A) edge (B);
\path[-] (B) edge (C);
\end{tikzpicture}
\end{tabular}}
\end{table}

In the first five cases, $B$ is a non-triplex node in $\rho'$ and, thus, $B \notin D(Z)$ for $\rho'$ to be $Z$-open. Since $Z$ contains no error node, $Z$ cannot determine any node in $V$ that is not already in $Z$. Then, $B \notin Z$. In the last case, $\epsilon^B$ is a non-triplex node in $\rho'$ and, thus, $\epsilon^B \notin D(Z)$ for $\rho'$ to be $Z$-open. Then, $B \notin Z$ or $pa_{G'}(B) \setminus \epsilon^B \setminus Z \ \neq \emptyset$. Then, $B \notin Z$ or $pa_{G}(B) \setminus Z \ \neq \emptyset$ (recall that $pa_{G'}(B)=pa_G(B) \cup \epsilon^B$ by construction of $G'$).
\end{proof}

\begin{theorem}\label{the:G'G'}
Assume that $G'$ has the same deterministic relationships no matter whether it is interpreted as an AMP or LWF CG. Then, $I_{AMP}(G')=I_{LWF}(G')$.
\end{theorem}

\begin{proof}
Assume for a moment that $G'$ has no deterministic node. Note that $G'$ has no induced subgraph of the form $A \ra B - C$ with $A, B, C \in V \cup \epsilon$. Such an induced subgraph is called a flag by \citet[pp. 40-41]{Anderssonetal.2001}. They also introduce the term biflag, whose definition is irrelevant here. What is relevant here is the observation that a CG cannot have a biflag unless it has some flag. Therefore, $G'$ has no biflags. Consequently, every probability distribution that is Markovian wrt $G'$ when interpreted as an AMP CG is also Markovian wrt $G'$ when interpreted as a LWF CG and vice versa \citep[Corollary 1]{Anderssonetal.2001}. Now, note that there are Gaussian probability distributions that are faithful to $G'$ when interpreted as an AMP CG \citep[Theorem 6.1]{Levitzetal.2001} as well as when interpreted as a LWF CG \citep[Theorems 1 and 2]{Penna2011}. Therefore, $I_{AMP}(G')=I_{LWF}(G')$. We denote this independence model by $I_{NDN}(G')$.

Now, forget the momentary assumption made above that $G'$ has no deterministic node. Recall that we assumed that $D(Z)$ is the same under the AMP and the LWF interpretations of $G'$ for all $Z \subseteq V \cup \epsilon$. Recall also that, from the point of view of the separations in an AMP or LWF CG, that a node is determined by the conditioning set has the same effect as if the node were in the conditioning set. Then, $X \ci_{G'} Y | Z$ is in $I_{AMP}(G')$ iff $X \ci_{G'} Y | D(Z)$ is in $I_{NDN}(G')$ iff $X \ci_{G'} Y | Z$ is in $I_{LWF}(G')$. Then, $I_{AMP}(G')=I_{LWF}(G')$.
\end{proof}

The first major result of this paper is the following corollary, which shows that every AMP CG is Markov equivalent to some LWF CG under marginalization. The corollary follows from Theorems \ref{the:GG'} and \ref{the:G'G'}.

\begin{corollary}\label{cor:GG'}
$I_{AMP}(G)=[I_{LWF}(G')]_\epsilon^\emptyset$.
\end{corollary}

Now, let $G''$ denote the DAG obtained from $G'$ by replacing every edge $\epsilon^A - \epsilon^B$ in $G'$ with $\epsilon^A \ra S_{\epsilon^A\epsilon^B} \la \epsilon^B$. Figure \ref{fig:example} shows an example. The nodes $S_{\epsilon^A\epsilon^B}$ are called selection nodes. Let $S$ denote all the selection nodes in $G''$. The following theorem relates the semantics of $G'$ and $G''$.

\begin{theorem}\label{the:G'G''}
Assume that $G'$ and $G''$ have the same deterministic relationships. Then, $I_{LWF}(G')=[I_{LWF}(G'')]_\emptyset^S$.
\end{theorem}

\begin{proof}
Assume for a moment that $G'$ has no deterministic node. Then, $G''$ has no deterministic node either. We show below that every $Z$-open route between $\alpha$ and $\beta$ in $G'$ can be transformed into a $(Z \cup S)$-open route between $\alpha$ and $\beta$ in $G''$ and vice versa, with $\alpha, \beta \in V \cup \epsilon$. This implies that $I_{LWF}(G')=[I_{LWF}(G'')]_\emptyset^S$. We denote this independence model by $I_{NDN}(G')$.

First, let $\rho'$ denote a $Z$-open route between $\alpha$ and $\beta$ in $G'$. Then, we can easily transform $\rho'$ into a $(Z \cup S)$-open route $\rho''$ between $\alpha$ and $\beta$ in $G''$: Simply, replace every edge $\epsilon^A - \epsilon^B$ in $\rho'$ with $\epsilon^A \ra S_{\epsilon^A\epsilon^B} \la \epsilon^B$. To see that $\rho''$ is actually $(Z \cup S)$-open, note that every collider section in $\rho'$ is due to a subroute of the form $A \ra B \la C$ with $A, B \in V$ and $C \in V \cup \epsilon$. Then, any node that is in a collider (respectively non-collider) section of $\rho'$ is also in a collider (respectively non-collider) section of $\rho''$.

Second, let $\rho''$ denote a $(Z \cup S)$-open route between $\alpha$ and $\beta$ in $G''$. Then, we can easily transform $\rho''$ into a $Z$-open route $\rho'$ between $\alpha$ and $\beta$ in $G'$: First, replace every subroute $\epsilon^A \ra S_{\epsilon^A\epsilon^B} \la \epsilon^A$ of $\rho''$ with $\epsilon^A$ and, then, replace every subroute $\epsilon^A \ra S_{\epsilon^A\epsilon^B} \la \epsilon^B$ of $\rho''$ with $\epsilon^A - \epsilon^B$. To see that $\rho'$ is actually $Z$-open, note that every undirected edge in $\rho'$ is between two noise nodes and recall that no noise node has incoming directed edges in $G'$. Then, again every collider section in $\rho'$ is due to a subroute of the form $A \ra B \la C$ with $A, B \in V$ and $C \in V \cup \epsilon$. Then, again any node that is in a collider (respectively non-collider) section of $\rho'$ is also in a collider (respectively non-collider) section of $\rho''$.

Now, forget the momentary assumption made above that $G'$ has no deterministic node. Recall that we assumed that $D(Z)$ is the same no matter whether we are considering $G'$ or $G''$ for all $Z \subseteq V \cup \epsilon$. Recall also that, from the point of view of the separations in a LWF CG, that a node is determined by the conditioning set has the same effect as if the node were in the conditioning set. Then, $X \ci_{G''} Y | Z$ is in $[I_{LWF}(G'')]_\emptyset^S$ iff $X \ci_{G'} Y | D(Z)$ is in $I_{NDN}(G')$ iff $X \ci_{G'} Y | Z$ is in $I_{LWF}(G')$. Then, $I_{LWF}(G')=[I_{LWF}(G'')]_\emptyset^S$.
\end{proof}

The second major result of this paper is the following corollary, which shows that every AMP CG is Markov equivalent to some DAG under marginalization and conditioning. The corollary follows from Corollary \ref{cor:GG'}, Theorem \ref{the:G'G''} and the fact that $G''$ is a DAG and, thus, $I_{AMP}(G'') = I_{LWF}(G'')$.

\begin{corollary}\label{cor:GG''}
$I_{AMP}(G)=[I_{LWF}(G'')]_\epsilon^S=[I_{AMP}(G'')]_\epsilon^S$.
\end{corollary}

\section{EAMP CGs Are Closed under Marginalization}\label{sec:closed}

In this section, we show that EAMP CGs are closed under marginalization, meaning that for any EAMP CG $G'$ and $L \subseteq V$ there is an EAMP CG $[G']_L$ st $[I_{AMP}(G')]_{L \cup \epsilon}=[I_{AMP}([G']_L)]_\epsilon$. We actually show how to transform $G'$ into $[G']_L$.

To gain some intuition into the problem and our solution to it, assume that $L$ contains a single node $B$. Then, marginalizing out $B$ from the system of linear equations associated with $G$ implies the following: For every $C$ st $B \in pa_{G}(C)$, modify the equation $C = \beta_C \: pa_{G}(C) + \epsilon^C$ by replacing $B$ with the right-hand side of its corresponding equation, i.e. $\beta_B \: pa_{G}(B) + \epsilon^B$ and, then, remove the equation $B = \beta_B \: pa_{G}(B) + \epsilon^B$ from the system. In graphical terms, this corresponds to $C$ inheriting the parents of $B$ in $G'$ and, then, removing $B$ from $G'$. The following pseudocode formalizes this idea for any $L \subseteq V$.

\begin{table}[H]
\centering
\scalebox{1.0}{
\begin{tabular}{ll}
1 & Let $[G']_L=G'$\\
2 & Repeat until all the nodes in $L$ have been considered\\
3 & \hspace{0.3cm} Let $B$ denote any node in $L$ that has not been considered before\\
4 & \hspace{0.3cm} For each pair of edges $A \ra B$ and $B \ra C$ in $[G']_L$ with $A, C \in V \cup \epsilon$\\
5 & \hspace{0.8cm} Add the edge $A \ra C$ to $[G']_L$\\
6 & \hspace{0.3cm} Remove $B$ and all the edges it participates in from $[G']_L$\\
\end{tabular}}
\end{table}

Note that the result of the pseudocode above is the same no matter the ordering in which the nodes in $L$ are selected in line 3. Note also that we have not yet given a formal definition of EAMP CGs. We define them recursively as all the graphs resulting from applying the pseudocode in the previous section to an AMP CG, plus all the graphs resulting from applying the pseudocode in this section to an EAMP CG. It is easy to see that every EAMP CG is an AMP CG over $W \cup \epsilon$ with $W \subseteq V$ and, thus, its semantics are defined. Theorem \ref{the:GG'} together with the following theorem confirm that these semantics are as desired.

\begin{theorem}\label{the:closed}
$[I_{AMP}(G')]_{L \cup \epsilon}=[I_{AMP}([G']_L)]_\epsilon$.
\end{theorem}

\begin{proof}
We find it easier to prove the theorem by defining separation in AMP CGs in terms of routes rather than paths. A node $B$ in a route $\rho$ in an AMP CG $G$ is called a triplex node in $\rho$ if $A \ra B \la C$, $A \ra B - C$, or $A - B \la C$ is a subroute of $\rho$ (note that maybe $A=C$ in the first case). A node $B$ in $\rho$ is called a non-triplex node in $\rho$ if $A \la B \ra C$, $A \la B \la C$, $A \la B - C$, $A \ra B \ra C$, $A - B \ra C$, or $A - B - C$ is a subroute of $\rho$ (note that maybe $A=C$ in the first and last cases). Note that $B$ may be both a triplex and a non-triplex node in $\rho$. Moreover, $\rho$ is said to be $Z$-open with $Z \subseteq V$ when 

\begin{itemize}
\item every triplex node in $\rho$ is in $D(Z)$, and

\item no non-triplex node in $\rho$ is in $D(Z)$.
\end{itemize}

When there is no $Z$-open route in $G$ between a node in $X$ and a node in $Y$, we say that $X$ is separated from $Y$ given $Z$ in $G$ and denote it as $X \ci_G Y | Z$. This and the standard definition of separation in AMP CGs introduced in Section \ref{sec:preliminaries} are equivalent, in the sense that they identify the same separations in $G$ \citep[Remark 3.1]{Anderssonetal.2001}.

We prove the theorem for the case where $L$ contains a single node $B$. The general case follows by induction. Specifically, given $\alpha, \beta \in V \setminus L$ and $Z \subseteq V \setminus L \setminus \alpha \setminus \beta$, we show below that every $Z$-open route between $\alpha$ and $\beta$ in $[G']_L$ can be transformed into a $Z$-open route between $\alpha$ and $\beta$ in $G'$ and vice versa.

First, let $\rho$ denote a $Z$-open route between $\alpha$ and $\beta$ in $[G']_L$. We can easily transform $\rho$ into a $Z$-open route between $\alpha$ and $\beta$ in $G'$: For each edge $A \ra C$ or $A \la C$ with $A, C \in V \cup \epsilon$ that is in $[G']_L$ but not in $G'$, replace each of its occurrence in $\rho$ with $A \ra B \ra C$ or $A \la B \la C$, respectively. Note that $B \notin D(Z)$ because $\epsilon^B \notin Z$.

Second, let $\rho$ denote a $Z$-open route between $\alpha$ and $\beta$ in $G'$. Note that $B$ cannot participate in any undirected edge in $G'$, because $B \in V$. Note also that $B$ cannot be a triplex node in $\rho$, because $B \notin D(Z)$. Note also that $B \neq \alpha, \beta$. Then, $B$ can only appear in $\rho$ in the following configurations: $A \ra B \ra C$, $A \la B \la C$, or $A \la B \ra C$ with $A, C \in V \cup \epsilon$. Then, we can easily transform $\rho$ into a $Z$-open route between $\alpha$ and $\beta$ in $[G']_L$: Replace each occurrence of $A \ra B \ra C$ in $\rho$ with $A \ra C$, each occurrence of $A \la B \la C$ in $\rho$ with $A \la C$, and each occurrence of $A \la B \ra C$ in $\rho$ with $A \la \epsilon^B \ra C$. In the last case, note that $\epsilon^B \notin D(Z)$ because $B \notin Z$.
\end{proof}

\section{Conclusions}\label{sec:conclusions}

In this paper, we have introduced EAMP CGs to model explicitly the errors in the system of linear equations associated to an AMP CG. We have shown that, as desired, every AMP CG is Markov equivalent to its corresponding EAMP CG under marginalization. We have used this result to show that every AMP CG is Markov equivalent to some LWF CG under marginalization. This result links the two most popular interpretations of CGs. We have used the previous result to show that every AMP CG is also Markov equivalent to some DAG under marginalization and conditioning. This result implies that the independence model represented by an AMP CG can be accounted for by some data generating process that is partially observed and has selection bias. Finally, we have shown that EAMP CGs are closed under marginalization, which guarantees parsimonious models under marginalization.

We are currently studying the following two questions. Can we modify EAMP CGs so that they become closed under conditioning too ? Can we repeat the work done here for LWF CGs ? That is, can we add deterministic nodes to LWF CGs so that they have DAGs as departure point and they become closed under marginalization and conditioning ? 

\section*{Acknowledgments}

This work is funded by the Center for Industrial Information Technology (CENIIT) and a so-called career contract at Link\"oping University, by the Swedish Research Council (ref. 2010-4808), and by FEDER funds and the Spanish Government (MICINN) through the project TIN2010-20900-C04-03.

\end{document}